\documentclass[11pt]{article}

\usepackage[utf8]{inputenc}
\usepackage[T1]{fontenc}
\usepackage{lmodern}
\usepackage{tikz}
\usepackage{pdfpages}
\usetikzlibrary{calc,shapes.geometric,arrows,positioning}
\usepackage{amsmath,amssymb,amsthm,mathtools}
\usepackage{bm}
\usepackage{bbm}
\usepackage{natbib}
\usepackage{geometry}
\usepackage{graphicx}
\usepackage{subcaption}
\usepackage{hyperref}
\usepackage{booktabs}
\hypersetup{
    colorlinks=true,
    linkcolor=blue,
    citecolor=blue,
    urlcolor=blue
}

\usepackage{enumitem}

\newtheorem{theorem}{Theorem}

\newtheorem{corollary}{Corollary}

\theoremstyle{definition}
\newtheorem{definition}{Definition}
\newtheorem{example}{Example}
\newtheorem{remark}{Remark}
\newtheorem{prop}{Proposition}

\newcommand{\R}{\mathbb{R}}

\title{
A Bregman Perspective on Classification and Regression Trees
}

\author{
Mathias Bourel \\ \\
IESTA, Facultad de Ciencias Económicas y de Administración,\\ Universidad de la República, Uruguay\\ 
IRL-2030, Instituto Franco-Uruguayo de Matemática e Interacciones (IFUMI)\\
\texttt{mathias.bourel@fcea.edu.uy}
}

\date{\today}

\begin{document}

\maketitle

\begin{abstract}
Classification and Regression Trees (CART) constitute one of the most influential paradigms in statistical learning. Although a variety of impurity measures have been proposed for different statistical models, these criteria are typically introduced on a case-by-case basis and analyzed separately. In this paper, we study CART through the lens of Bregman divergences. This perspective places the classical least-squares criterion, Poisson deviance, Kullback–Leibler-type losses, and other impurity measures associated with exponential-family models within a common framework. As a result, key ingredients of the CART methodology—including node representatives, impurity measures, and split-selection rules- can be expressed and analyzed through general properties of convex functions rather than through separate model-specific constructions. Beyond the algorithmic formulation, we investigate theoretical properties of Bregman-based CART procedures. In particular, we analyze how geometric properties of the generating convex function influence impurity reductions and stability of recursive partitions. We also establish consistency results within the proposed framework, providing a unified theoretical treatment for a broad family of CART-type procedures. Our results provide a geometric interpretation of impurity-based tree construction and show that many classical CART impurity criteria admit a common interpretation within a Bregman framework.

\end{abstract}

\section{Introduction}

Decision trees are among the most widely used tools in statistical learning due to their interpretability and flexibility. The classical CART (Classification and Regression Trees) algorithm, introduced by Breiman, Friedman, Olshen, and Stone in 1984 (\cite{breiman1984classification}), became one of the most influential and widely used methods for classification and regression problems. When the goal is to predict a continuous response, regression trees are
typically based on squared-error loss, whereas classification trees rely
on criteria such as the misclassification error, entropy, or the Gini
index.

In this work, we develop a general Bregman-divergence framework for CART in both regression and classification settings. Bregman divergences, introduced by Lev Bregman in 1967 \cite{Bregman1967} in the context of convex programming, have since become fundamental tools in optimization, information geometry, and statistical learning. For two points $x$ and $y$ in a convex set $C$,  and a strictly convex, differentiable function $\psi:C \to \mathbb{R}$ the Bregman divergence generated by $\psi$ between $x$ and $y$ is defined as:
\[
D_\psi(x,y)
=
\psi(x)-\psi(y)-\langle \nabla\psi(y),x-y\rangle.
\] 

This perspective has had a significant impact in clustering and related problems  (\cite{banerjee2005clustering}), but despite these connections, a systematic treatment of CART impurity measures and split gains within a general Bregman framework has not been fully developed for recursive partitioning. While the basic identities of Bregman divergences are well known, their implications for recursive partitioning, impurity reduction, and the analysis of CART-type procedures have received comparatively less attention than in clustering and related settings.  Although many of these criteria can be related to exponential-family models and Bregman divergences, the underlying geometric structure is often left implicit, particularly in the analysis of impurity reduction and split selection. Moreover, in contrast to clustering, which seeks a global partition minimizing a given objective, decision trees construct partitions recursively using greedy splitting criteria. This sequential structure introduces new analytical challenges and requires a dedicated study of the split gain as a central object.


Several extensions of recursive partitioning methods have been proposed for
responses arising from exponential-family models. In particular,
model-based recursive partitioning and generalized linear model trees
fit parametric models within each terminal node and recursively split the data
according to parameter instability tests
\cite{zeileis2008model,hothorn2015partykit,seibold2016model}.
More recently, distributional trees have extended this framework by
allowing all parameters of a distributional model to vary across the
partition \cite{schlosser2019distributional}.
These approaches focus on local likelihood modeling and parameter
heterogeneity across subpopulations. In contrast, the present work
retains the classical CART paradigm of piecewise-constant prediction
within terminal nodes and studies CART impurity measures from a unified Bregman perspective.  Our contribution is to show that many impurity criteria used within the CART paradigm can be formulated and analyzed through a common Bregman-divergence framework. Building on this perspective, we formulate node impurity and split selection in terms of Bregman divergences.

The present work makes this structure explicit within the CART framework. By formulating node impurity and split selection in terms of Bregman divergences, we obtain a unified perspective that covers impurity measures associated with a broad class of exponential-family models. The main contributions of the paper
are threefold. First, we formulate
CART impurity measures and split gains within a general Bregman
framework. Second, we study impurity reduction from a geometric perspective and relate the resulting split gains to classical deviance-based criteria. Third, we investigate the resulting Bregman Trees through
theoretical analysis and empirical studies.
 The remainder of the paper is organized as follows. Section 2 reviews the necessary background on exponential families and Bregman divergences. Section 3 provides a brief review of the CART methodology, including impurity measures, splitting criteria, and pruning procedures. Section 4  introduces the proposed Bregman Tree methodology and Section 5 discusses its theoretical properties. Section 6 reports simulation studies and a real-data application comparing Bregman Trees and CART. Finally, we conclude with a discussion and directions for future research.

%

\section{Background on Bregman Divergences and Exponential Families
}
The construction of the proposed Bregman Trees relies on a close relationship between exponential-family distributions, convex analysis, and Bregman divergences. The aim of this section is to review these connections and establish the notation used throughout the paper. We start with the definition of Bregman divergences and some associated properties, then  we present the exponential-family framework, and the associated convex conjugate functions. These concepts will serve as the basis for the impurity measures considered later in the tree-building procedure.

\subsection{Bregman Divergences and some basic properties}

Bregman divergences constitute a broad class of discrepancy measures generated by convex functions. They encompass many well-known loss functions used in statistics, machine learning, and information theory, including the squared loss, the Kullback--Leibler divergence, the logistic loss, the Itakura--Saito divergence, and the Mahalanobis distance.  They were introduced by Lev Bregman in 1967  in the context of convex programming (\cite{Bregman1967}). Since then, they have found numerous applications in areas such as information geometry, statistical learning, and convex optimization. Their importance stems from the fact that many statistical models and loss functions can be naturally expressed in terms of a suitable Bregman divergence. We begin by recalling the formal definition of a Bregman divergence and some of its fundamental properties.

\begin{definition}
Let $\psi:C\to\R$ be a differentiable strictly convex function on $C \subset \mathbb{R}^n$ a convex set.
The associated Bregman divergence $D_{\psi}:C \times C \to [0,+\infty)$ is defined as

$$D_\psi(x,y)
=
\psi(x)-\psi(y)-\langle \nabla\psi(y),x-y\rangle.$$
\end{definition}

\begin{example}
\begin{enumerate}
\item  If $\psi(x)=x^2$ then $D_{\psi}(x,y)=(x-y)^2$ corresponds to the euclidean distance.
\item If ${\bf p}=(p_1,\dots,p_n)\geq {\bf 0} , \sum p_i=1$ is a vector of probability and $\psi({\bf p})=\sum \limits_{i} p_i \log(p_i)$ then $D_{\psi}({\bf p},{\bf q})=\sum \limits_{i} p_i \log \frac{p_i}{q_i}=D_{KL}({\bf p}||{\bf q}) $ is the Kullback-Leibler divergence between vectors ${\bf p}$ and ${\bf q}$).
\end{enumerate}
\end{example}


The Bregman divergence is exactly the error of the approximation between $\psi(x)$ and the linear approximation at $\psi(y)$ given by its tangent plane:
\[
D_\psi(x,y) = \psi(x) - \big[ \psi(y) +  \langle \nabla\psi(y), x - y\rangle \big].
\] Another interpretation can be given by the approximation: assuming $\psi$ have derivatives of second order   \[
D_\psi(x, y) \approx \frac{1}{2}(x - y)^t \nabla^2 \psi(\zeta) (x - y)
\] where $\zeta$ is a point between $x$ and $y$ so the Bregman divergence corresponds with
the local curvature of $\psi$.

\begin{figure}[h]
  \centering
  \begin{subfigure}[b]{0.45\textwidth}
    \centering
    \includegraphics[width=\textwidth]{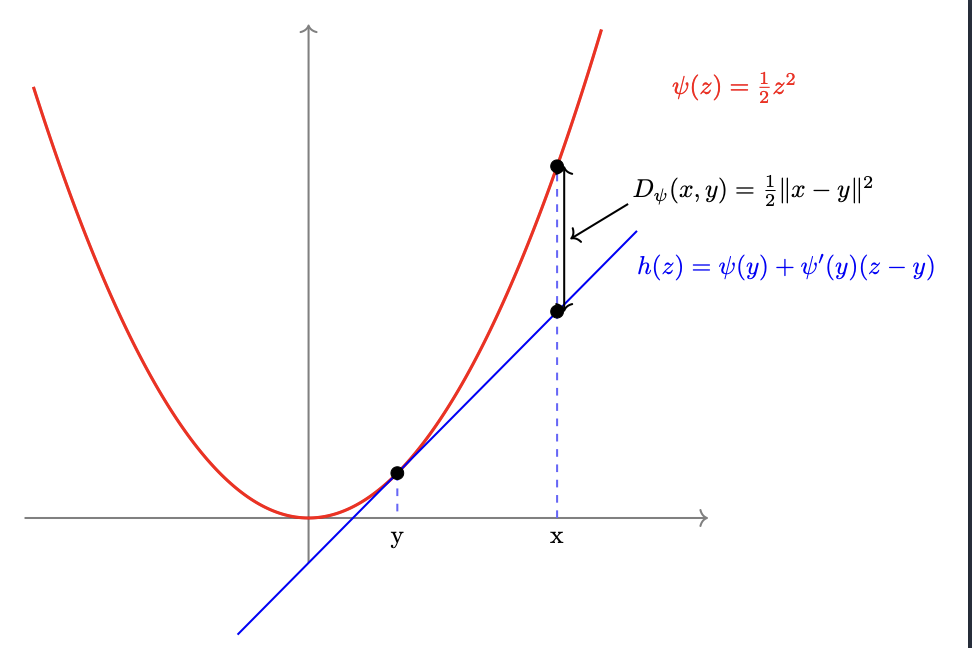}
    \label{fig:sqrdBregman}
  \end{subfigure}
  \hfill
  \begin{subfigure}[b]{0.45\textwidth}
    \centering
    \includegraphics[width=\textwidth]{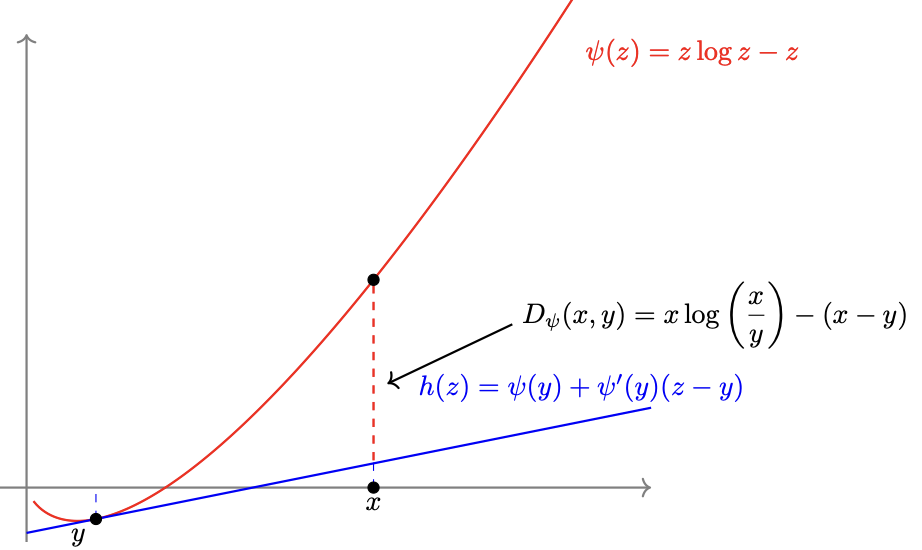}
    \label{fig:poissBregman}
  \end{subfigure}
  \caption{Examples of Bregman divergences. Left: squared Euclidean distance generated by $\psi(x)=x^2$. Right: Poisson divergence generated by $\psi(x)=x\log x-x$. In both cases, the divergence is represented by the difference between the convex function and its tangent line at the reference point.}

\end{figure}

Unlike distances, Bregman divergences are generally not symmetric and do not satisfy the triangle inequality. They are strictly convex in 1st argument but in general not in 2nd argument. However, they possess several important properties that make them particularly suitable for statistical learning, and have been established, for example in \cite{pfau2025generalized}:

\begin{prop}
    Let $\psi:C\to\R$ be a differentiable strictly convex function on $C \subset \mathbb{R}^n$ a convex set and $D_{\psi}:C \times C \to [0,+\infty)$ its associated Bregman divergence. If $X$ is a random variable, then:
    \begin{enumerate}
        \item For all $x,y \in C$: $D_{\psi}(x,y)\geq 0$ and equals 0 iff $x=y$.
        \item  $\mathbb{E}(X)=\underset{z}{\text{Argmin}}\,\mathbb{E}\bigl(D_{\psi}(X,z)\bigr)$
 
\item {\bf Bias-variance decomposition:} $\underbrace{\mathbb{E}\bigl(D_{\psi}(X,s)\bigr)}_{\text{MSE}}=\underbrace{D_{\psi}\bigl(\mathbb{E}(X),s \bigr)}_{\text{Bias}^2}+\underbrace{\mathbb{E}\bigl(D_{\psi}(X,\mathbb{E}(X))\bigr)}_{\text{Var}}$ 
\item {\bf The three points property:}  $D_{\psi}(x,z)=D_{\psi}(x,y) + D_{\psi}(y,z)+\langle \nabla \psi(z)-\nabla \psi(y),x-y\rangle$.

This is a generalization of law of cosinus:
    because with $\psi(x)=\frac{1}{2}||x||^2$, the euclidean case gives $$||x-z||_2^{2}=||x-y||_2^{2}+||y-z||_2^{2}-2 \underbrace{\langle z-y,x-y\rangle}_{||z-y||_2^{2}||x-y||_2^{2}\cos \angle xyz}$$

    \end{enumerate}
\end{prop}














Bregman divergences have steadily moved from a theoretical curiosity to a practical
building block across a wide range of modern machine learning and optimisation problems.
In the optimisation literature, they underpin Mirror Descent
\citep{nemirovsky1983wiley,raskutti2015information}, a generalisation of gradient descent that replaces
the Euclidean projection step with a geometry-aware update tailored to the structure of
the parameter space. More recently, Bregman divergences have been incorporated into
Optimal Transport \citep{kainth2023bregman}, where the standard squared-Euclidean cost
 is replaced by a divergence-based transport cost,
enabling richer comparisons between probability measures, transforming the usual Wasserstein distance  $W_c(\mu,\nu) = \inf_{\gamma \in \Pi(\mu,\nu)} 
\int c(x,y)\, d\gamma(x,y)$ in
 $W_{D_{\psi}}(\mu,\nu)=\inf_{\gamma \in \Pi(\mu,\nu)} \int_{}D_{\psi}(x,y)\,d\gamma(x,y)$. In the context of unsupervised
learning, \citet{banerjee2005clustering} showed that the $K$-means algorithm generalises naturally
to any Bregman divergence: replacing the squared Euclidean distance with $D_\psi(x,\mu_k)$
yields a family of clustering algorithms adapted to the geometry of the data. In that setting, cluster centroids are characterized as minimizers of Bregman risk, remplacing the minimization of $\sum \limits_{k=1}^{K} \sum \limits_{x \in C_k}||x-\mu_k||^2$ by $\underset{C_1,\dots,C_K}{\text{Argmin}}\sum \limits_{k=1}^{K} \sum \limits_{x \in C_k}D_{\psi}(x,\mu_k)$, and a decomposition of the total divergence analogous to the classical bias--variance identity is established. These results provide a unified view of clustering algorithms such as $k$-means and reveal deep connections between geometry and statistical modeling. Finally,
at the intersection of ensemble methods and loss function theory, \citet{wood2023unified} established
a unified bias--variance--diversity decomposition for arbitrary Bregman divergences,
extending the classical squared-error ambiguity decomposition to the full exponential family.
Taken together, these developments suggest that Bregman divergences provide a unifying
language for a surprisingly broad set of problems, and that results derived in one
setting---such as the tree-based methods studied here---may carry implications well
beyond their original domain.

\subsection{Bregman divergence and exponential family}

Bregman divergences have been extensively used in statistics, information theory, and machine learning as a unifying framework for a broad class of loss functions associated with convex potentials. The main reason lies in their connection with exponential family distributions where it has been well established: the maximum likelihood estimation can be equivalently formulated as the minimization of a Bregman divergence (\cite{banerjee2005clustering}).
Exponential families play a central role in modern statistical modeling and constitute the probabilistic foundation of generalized linear models (GLMs). In the GLM framework, the response distribution is assumed to belong to an exponential family, allowing a unified treatment of a wide range of data types, including Gaussian, Poisson, binomial, and Gamma responses (\cite{McCullagh1989}).

A random variable $X$ has a distribution coming from the one parameter exponential family if its density function can be written as:

$$p_{\theta}(x)=\exp\bigl(\theta x - \psi(\theta)\bigr)p_0(x)$$

where $\theta$ is called the natural parameter and $\psi(\theta)$ is called the log normalizer or cumulant function.   Many distributions come from the exponential family: Bernoulli, Binomial, Poisson (discrete), Exponential, Beta, Gaussian (continuous), etc.
\begin{example}
    
\begin{enumerate}
\item {\bf Poisson density.}  $\mathbb{P}(X=x)=\frac{e^{-\lambda}{\lambda}^{ x}}{x!}  =e^{x \log\lambda-\lambda-\log(x!)}, \,\,\theta=\log (\lambda), \psi(\theta)=e^{\theta}$

\item {\bf Bernoulli density.} $\mathbb{P}(X=x)=p^{x}(1-p)^{1-x} =e^{\log(p^x(1-p)^{1-x})}=e^{\left( x \theta -\log(1+e^{\theta}) \right)},$
$ \theta=\log\left( \frac{p}{1-p}\right),\,\psi(\theta)=\log(1+e^{\theta})$
\item {\bf Normal density.} $f_{\mu}(x)=\frac{1}{\sqrt{2\pi}}e^{\frac{1}{2}(x-\mu)^2}  =e^{\mu x-\frac{\mu^2}{2}}\frac{1}{\sqrt{2\pi}}e^{-x^2/2}$,
$ \theta=\mu,\,\psi(\theta)=\frac{\theta^2}{2}$

\end{enumerate}

\end{example}

The following review is necessarily brief and focuses only on the aspects of exponential families that are relevant to the development of the proposed methodology. Readers interested in a more comprehensive treatment are referred to Efron's monograph \emph{Exponential Families in Theory and Practice} (\cite{efron2022exponential}), which provides an extensive account of the theoretical foundations and statistical applications of exponential-family models.

 \medskip

For any convex function $f$ with domain $S$ the gradient of $f$ at some point $x$ is a vector $v$ that satisfies $f(y)\geq f(x)+\langle v, y-x \rangle\,\,\forall\,y\in S$
which is equivalent to:
$\underset{y\in S}{\max}\bigl\{\langle v,y \rangle -f(y)\bigr\}=\langle v, x \rangle -f(x)$. This motivates the following definition of convex conjugate Legendre function:

\begin{definition}
    Let $\psi:C \subset \mathbb{R}^n \to \overline{\mathbb{R}}$ a convex function. Its convex conjugate is the function $\psi^{*}(v):\mathbb{R}^n\to \overline{\mathbb{R}}$ given by $$\psi^{*}(v)=\underset{x\in C}{\sup}\,\left[\langle v,x \rangle -\psi(x)\right]$$ 
\end{definition}
The number $\psi^*(v)$ indicates how far it is needed to lower a line with slope $v$ so that it touches but does not cross $\psi$.
   
 \begin{figure}[!ht]
 \begin{center}
              \includegraphics[scale=0.5]{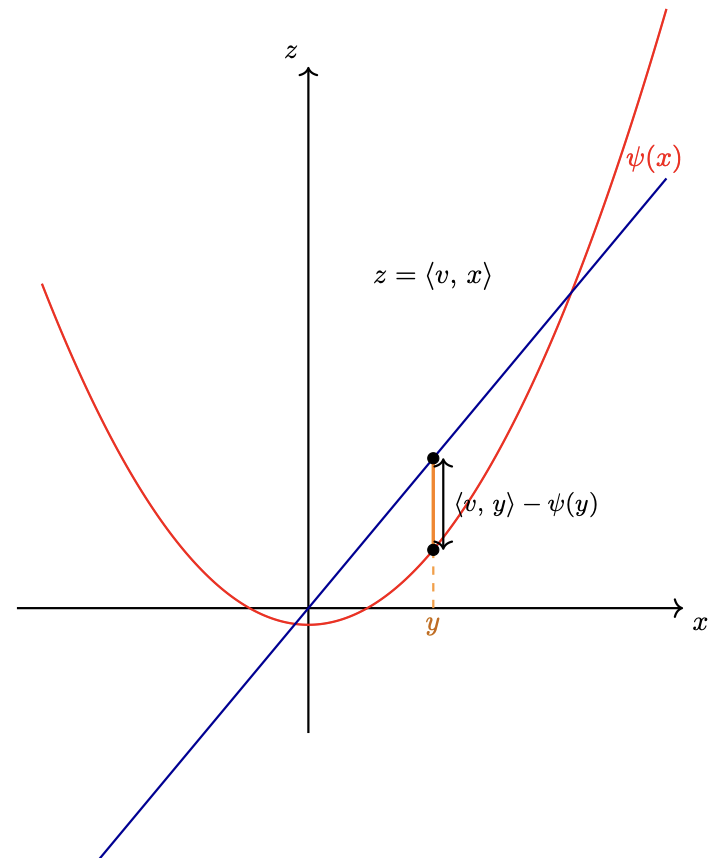}
              \caption{Convex Legendre conjugate of a convex function $\psi$}
               \end{center}
          \end{figure} 

\begin{example}
    
\begin{enumerate}
    \item If $\psi(x)=\langle a,x\rangle -b$, then $\psi^{*}(v)=\left\{\begin{array}{cc}b,& v=a\\ +\infty,&v\neq a \end{array}\right.$

   \item If $\psi(x)=\frac{1}{2}||x||^2$, then $\psi^{*}(v)=\frac{1}{2}||x||^2$ 
\end{enumerate}

\end{example}    
    
Let $\psi$ denote the cumulant function of a regular exponential family. The Legendre--Fenchel conjugate of $\psi$ is defined as
$$\psi^{*}(v)=
\sup_{\theta}\left\{
\langle v,\theta\rangle-\psi(\theta)
\right\}.$$
Since $\psi$ is differentiable and strictly convex, the supremum is attained at the unique point satisfying $v=\nabla\psi(\theta)$. Consequently, $\psi^{*}(v)
=\langle v,\theta\rangle-\psi(\theta)$, where $\theta$ is implicitly determined by the relation $v=\nabla\psi(\theta)$. In the context of exponential families, the gradient of the cumulant function coincides with the mean parameter \cite{efron2022exponential},
$\mu=\nabla\psi(\theta)$, so that the conjugate function can be expressed in terms of $\mu$ as $$
\psi^{*}(\mu) =
\bigl\langle \mu,\theta(\mu)\bigr\rangle -
\psi\bigl(\theta(\mu)\bigr).$$
Throughout the remainder of the paper, we shall denote the convex conjugate of $\psi$ by $\phi=\psi^{*}$. The function $\phi$ plays a central role in the theory of Bregman divergences, since the divergences associated with exponential-family models arise naturally from this convex dual representation. Under standard regularity conditions, the cumulant function $\psi$ is strictly convex and differentiable. Consequently, the gradient mappings $\nabla\psi$ and $\nabla\phi$, are inverses of each other. This establishes a one-to-one correspondence between the natural parameter space and the mean parameter space, $\mu=\nabla\psi(\theta),\,
\theta=\nabla\phi(\mu)$,
a property that plays a central role in the connection between exponential families and Bregman divergences. The function $\phi$ is itself convex, and the Legendre transform is involutive, namely $\phi^{*}=\psi$. A fundamental property of convex conjugates is Fenchel's inequality, $\langle x,y\rangle
\leq
\psi(x)+\phi(y)$, with equality if and only if $y=\nabla\psi(x)$. In particular,

\[
\psi(x)+\phi(\nabla\psi(x))
=
\langle x,\nabla\psi(x)\rangle.
\]

\medskip

A fundamental connection between exponential families and Bregman divergences was established by Banerjee et al (\cite{banerjee2005clustering}). Their result shows that every regular exponential-family distribution admits a representation in terms of a Bregman divergence generated by the Legendre conjugate of its cumulant function.

\begin{theorem}[\cite{banerjee2005clustering}]
Let $p_{\psi,\theta}$ be a member of a regular exponential family with cumulant function $\psi$, and let $\phi=\psi^*$ denote its Legendre conjugate. Then the density can be expressed as

\[
p_{\psi,\theta}(x)
=
\exp\!\left(
-D_{\phi}(x,\mu(\theta))
\right)
\, b_{\phi}(x),
\]

where $\mu(\theta)=\nabla\psi(\theta)$ is the mean parameter, $D_{\phi}$ is the Bregman divergence generated by $\phi$, and $b_{\phi}$ is a function depending only on $x$.
\end{theorem}

An important consequence of this result is that, for observations
$x_1,\ldots,x_n$, maximizing the likelihood is equivalent to minimizing the average Bregman divergence $\frac{1}{n}\sum_{i=1}^{n}
D_{\phi}(x_i,\mu)$, whose minimizer is the sample mean. This result provides the theoretical link between likelihood-based estimation in exponential families and impurity measures based on Bregman divergences.

The correspondence established by Banerjee et al. is illustrated in Table \ref{tab:banerjee_examples}, which summarizes several classical exponential-family distributions together with their natural parameters, cumulant functions, convex conjugates, and the associated Bregman divergences. Many commonly used loss functions arise as special cases of this general framework. Moreover, if $\psi(\theta)$ denotes the cumulant function of a regular exponential family and $\phi$  is its Legendre conjugate, then the corresponding unit deviance is exactly twice the Bregman divergence generated by the associated $\phi$. Hence, GLM deviance criteria are proportional to Bregman divergences generated by the convex conjugate of the cumulant function. Therefore, all standard GLM deviance criteria arise as particular instances of the same convex construction.

\begin{table}[ht]
\centering
\scriptsize

\begin{tabular}{|l|c|c|c|c|c|c|}
\hline
Distribution &
$p_{\theta}(y)$ &
$\theta$ &
$\psi(\theta)$ &
$\mu$ &
$\phi(\mu)$ &
$D_{\phi}(y,\mu)$
\\ \hline

Normal &
$\frac{1}{\sqrt{2\pi}}
e^{-\frac12 (y-\mu)^2}$ &
$\mu$ &
$\frac{\theta^2}{2}$ &
$\mu$ &
$\frac{\mu^2}{2}$ &
$\frac12 (y-\mu)^2$
\\ \hline

Bernoulli &
$p^y(1-p)^{1-y}$ &
$\log\!\left(\frac{p}{1-p}\right)$ &
$\log(1+e^\theta)$ &
$p$ &
$\mu\log\mu+(1-\mu)\log(1-\mu)$ &
$y\log\!\left(\frac{y}{\mu}\right)
+(1-y)\log\!\left(\frac{1-y}{1-\mu}\right)$
\\ \hline

Poisson &
$\frac{e^{-\lambda}\lambda^y}{y!}$ &
$\log\lambda$ &
$e^\theta$ &
$\lambda$ &
$\mu\log\mu-\mu$ &
$y\log\!\left(\frac{y}{\mu}\right)
-(y-\mu)$
\\ \hline

Exponential &
$\lambda e^{-\lambda y}$ &
$-\lambda$ &
$-\log(-\theta)$ &
$1/\lambda$ &
$-\log\mu-1$ &
$\frac{y}{\mu}
-\log\!\left(\frac{y}{\mu}\right)-1$
\\ \hline

\end{tabular}
\caption{Examples of exponential-family distributions and their associated Bregman divergences. For each model, the table reports the natural parameter $\theta$, the cumulant function $\psi$, its Legendre conjugate $\phi=\psi^*$, and the corresponding divergence.}
\label{tab:banerjee_examples}
\end{table}

Table \ref{tab:banerjee_examples} highlights that familiar losses such as the squared loss, the Kullback–Leibler divergence, the Poisson deviance, and the Itakura–Saito divergence all arise from a common convex-analytic construction. This observation provides the foundation for the unified treatment developed in the remainder of the paper.

\section{Classification and Regression Trees}

Decision trees are a fundamental class of nonparametric methods for supervised learning, with the CART framework introduced by Breiman et al.~\cite{breiman1984classification} being one of the most widely used algorithms in machine learning. The popularity of CART stems from several attractive features. First, decision trees are highly interpretable, since the resulting model can be represented as a sequence of simple decision rules. Second, they are nonparametric and therefore require few assumptions about the underlying relationship between predictors and the response. Third, CART naturally handles both numerical and categorical predictors and is able to capture complex nonlinear interactions between variables. Finally, tree-based methods are computationally efficient and provide the foundation for many successful ensemble techniques, particularly random forests and boosting methods.

 In the regression setting, the objective is to approximate the regression function

\[
f^*(x)
=
\arg\min_{z\in\mathbb R}
\mathbb E[(Y-z)^2\mid X=x]
=
\mathbb E[Y\mid X=x]
\]
by a piecewise constant predictor. A tree induces a partition
$\{t_1,\ldots,t_M\}$ of the feature space and defines $f_T(x)
=
\sum_{j=1}^M
c_j\,\mathbf 1_{\{x\in t_j\}}$, where each region corresponds to a terminal node of the tree. For a fixed region $t_j$, the optimal constant for a regression problem is $c_j
=
\arg\min_c
\mathbb E[(Y-c)^2\mid X\in t_j]
=
\mathbb E[Y\mid X\in t_j]$ and in practice, this quantity is estimated by the sample mean $\hat c_j
=
\frac1{|t_j|}
\sum_{i:X_i\in t_j}
Y_i$.

The construction of the tree proceeds recursively. At each node $t$, a split is selected to maximize the impurity reduction,

\[
\Delta(t)
=
I(t)
-
\frac{|t_L|}{|t|}
I(t_L)
-
\frac{|t_R|}{|t|}
I(t_R),
\]
where $t_L$ and $t_R$ denote the child nodes. Classical choices for $I(t)$ include the Gini index 
$I(t) = \sum_{k=1}^K p_k (1 - p_k)$ in classification, the entropy 
$I(t) = -\sum_{k=1}^K p_k \log p_k$, and the within-node variance 
$I(t) = \frac{1}{|t|} \sum_{i \in t} (y_i - \bar{y}_t)^2$ in regression. 
Despite their empirical success, these criteria are typically introduced from 
heuristic or information-theoretic perspectives, and a unifying mathematical 
framework that captures their common structure explicit is less commonly emphasized in the literature. Many of them can be understood within a common geometric framework based on Bregman divergences. Table~\ref{tab:rpart_bregman} summarizes some examples corresponding to the methods implemented in \texttt{rpart}.

\begin{table}[ht]
\centering
\small

\begin{tabular}{lllll}
\hline
Method & Distribution & Generator $\phi$ & Divergence $D_\phi(y,\mu)$ & Criterion \\
\hline

anova
&
Gaussian
&
$\frac12 x^2$
&
$\frac12(y-\mu)^2$
&
Squared loss
\\[0.2cm]

class (entropy)
&
Bernoulli
&
$x\log x +(1-x)\log(1-x)$
&
$y\log\!\left(\frac{y}{\mu}\right)
+
(1-y)\log\!\left(\frac{1-y}{1-\mu}\right)$
&
Kullback--Leibler
\\[0.2cm]

class (Gini)
&
Bernoulli
&
$x^2$
&
$(y-\mu)^2,\quad y\in\{0,1\}$
&
Gini index
\\[0.2cm]

poisson
&
Poisson
&
$x\log x-x$
&
$y\log\!\left(\frac{y}{\mu}\right)
-(y-\mu)$
&
Poisson deviance
\\

\hline
\end{tabular}
\caption{Classical impurity measures used in CART and their interpretation as Bregman divergences.}\label{tab:rpart_bregman}
\end{table}

Table \ref{tab:rpart_bregman} reveals that many classical CART criteria can be interpreted through the lens of Bregman divergences. However, existing implementations are restricted to a relatively small collection of divergences and do not explicitly formulate tree construction as a general Bregman optimization problem. This observation motivates the development of the framework proposed in this paper. By replacing the quadratic loss with an arbitrary Bregman divergence, we obtain a unified methodology encompassing a broad class of exponential-family models while preserving the recursive partitioning structure of CART.

A fully grown tree typically achieves a very low training error but may suffer from overfitting. To control model complexity, CART employs a  pruning procedure based on the cost-complexity criterion. Let $\tilde{T}$ denote the set of terminal nodes of a tree \(T\). The empirical risk of the tree is defined as $R(T) =
\sum_{t\in\tilde{T}}
\frac{|t|}{n}\,I(t)$,
where \(I(t)\) denotes the impurity of node \(t\). The cost-complexity functional is then given by $
R_\alpha(T)
=
R(T)
+
\alpha |\tilde{T}|$, where \(\alpha\ge 0\) is a complexity parameter controlling the trade-off between goodness of fit and model complexity. Breiman et al.~\cite{breiman1984classification} showed that, as \(\alpha\) increases, there exists a finite nested sequence of subtrees $
T_0 \supset T_1 \supset \cdots \supset T_K$, where \(T_0\) is the maximal tree and each subtree is optimal for a range of increasing values of $\alpha$. 
In practice, the value of \(\alpha\) is selected by cross-validation. Let \(\widehat R_{CV}(\alpha)\) denote the cross-validated estimate of the prediction error associated with the subtree indexed by \(\alpha\), and let \(\widehat \sigma(\alpha)\) be its estimated standard error. If \(\alpha_{\min}\) minimizes the cross-validated error, the one-standard-error (1-SE) rule selects the largest value of \(\alpha\) satisfying

\[
\widehat R_{CV}(\alpha)
\le
\widehat R_{CV}(\alpha_{\min})
+
\widehat \sigma(\alpha_{\min}).
\]

Equivalently, among all trees whose estimated prediction error lies within one standard error of the minimum, the 1-SE rule chooses the smallest tree. This heuristic often produces substantially simpler models while maintaining predictive performance comparable to that of the tree with minimum cross-validated error.

From a statistical perspective, regression trees may also be viewed as local averaging estimators. We recall that, in the regression setting, CART is constructed to minimize the quadratic risk
$\mathcal{R}(f)
=
\mathbb{E}\big[(Y-f(X))^2\big]$, which is  minimized by
$f^*(x)
=
\mathbb{E}[Y\mid X=x]$. 
Note that if $A_1,\dots,A_n$ is a partition of the space generated from the sample
$D_n=\{(X_i,Y_i)\}_{i=1}^n$, the estimator $\hat f_{n}$ can be written as $\hat f_{n}(x)=\sum \limits_{i=1}^n W_{n,i}(x) Y_i$ where $A_n(x)$ denotes the cell containing $x$, $
W_{n,i}(x) = \frac{\mathbf{1}\{X_i \in A_n(x)\}}{N_n(x)}$,
$N_n(x) = \#\{i : X_i \in A_n(x)\}$. The weights $W_{n,i}$ are non-negative and sum to 1. This representation connects tree-based estimators with the theory of local averaging procedures developed by Stone (\cite{stone1977consistent}) and later by Devroye, Györfi and Lugosi (\cite{devroye1996feature}).

Unlike histogram rules based on predetermined partitions, the original CART algorithm does not admit a general consistency theorem. Nevertheless, consistency can be established under suitable assumptions on the sequence of partitions. 
A classical result due to Devroye, Györfi and Lugosi 
(\cite{devroye1996feature})  guarantees in this context the consistency of CART based on Stone's theorem ( \cite{stone1977consistent}). The principal restriction considered on the partition is that the weights $W_{n,i}$ depend only on $X_1,\ldots,X_n$ (and not on $Y_1,\ldots,Y_n$). Under squared-error loss, the theoretical risk and the risk with the estimator are defined as $R(f^*)=\min_f \mathbb{E}((Y-f(X))^2)$ and $$R(\hat{f}_n)= \mathbb{E}((Y-\hat f_n(X))^2)=\mathbb{E}_{D_n}(\mathbb{E}_{(X,Y)}((Y-\hat f_n(X))^2|D_n))$$
and the excess risk is $$R(\hat{f}_n)-R(f^*)=\mathbb{E}\big[\bigl(\hat f_n(X) - f^*(X)\bigr)^2\big]$$

\begin{theorem}[Devroye--Györfi--Lugosi]\label{DGLtheorem}
Let $(X,Y)\in \mathbb{R}^d \times \mathbb{R}$ and let $D_n=\{(X_i,Y_i)\}_{i=1}^n$ be an i.i.d. sample. Let $\mathcal{P}_n$ be a partition of $\mathbb{R}^d$ obtained from this sample. For $x \in \mathbb{R}^d$, let $A_n(x)$ be the cell of $\mathcal{P}_n$ containing $x$, and define
\[
\hat f_n(x) = \frac{1}{N_n(x)} \sum_{i: X_i \in A_n(x)} Y_i,
\quad
N_n(x) = \#\{i : X_i \in A_n(x)\}.
\]

Assume that:
\begin{enumerate}
    \item $\mathrm{diam}(A_n(x)) \xrightarrow{P} 0$ (Cell shrinkage),
    \item $N_n(x) \xrightarrow{P} \infty$ (Neighborhood growth),
    \item $\mathbb{E}[Y^2] < \infty$.
\end{enumerate}

Then
\[
\mathcal{R}(\hat f_n)-\mathcal{R}(f^*)=\mathbb{E}\big[\bigl(\hat f_n(X) - f^*(X)\bigr)^2\big] \to 0.
\]
\end{theorem}

\section{Bregman Trees}
Motivated by the correspondence between exponential families and Bregman divergences, we extend the CART framework by replacing the classical CART impurity measurewith a general Bregman impurity measure. This leads to a unified recursive partitioning methodology encompassing Gaussian, Poisson, Bernoulli, Gamma, Exponential, and many other models within a single convex-analytic framework. The resulting procedure retains the essential structure of CART while allowing the impurity criterion to be adapted to the underlying distributional assumptions.
Bregman Trees generalize the classical CART construction by replacing the impurity measure of CART with a Bregman divergence $D_\phi(y,z)=\phi(y)-\phi(z)-\langle \nabla\phi(z),\, y-z\rangle$ where \(\phi\) is a strictly convex differentiable function.  In the exponential family, two Bregman divergences arise naturally. The first, $D_\psi$, 
is generated by the log-partition function $\psi(\theta)$ and operates in the space of 
natural parameters; it measures the discrepancy between two distributions. The second, 
$D_\phi$, is generated by the Legendre-Fenchel conjugate $\phi = \psi^*$ and 
operates in the space of mean parameters $\mu = \nabla\psi(\theta)$. Since each leaf of 
the tree produces a prediction in the mean space, namely $\hat{\mu} = \bar{y}$, the 
appropriate loss for measuring prediction error is $D_\phi(y, \hat{\mu})$, which 
compares observed responses directly against predicted means. For regular exponential families, minimizing 
$\sum_i D_\phi(y_i, \hat{\mu})$ is equivalent to maximizing the log-likelihood (\cite{banerjee2005clustering}), establishing $D_\phi$ as the canonical prediction loss for exponential family 
regression. 

The associated population risk is\[\mathcal{R}_\phi(f)=\mathbb{E}\big[D_\phi(Y,f(X))\big],\] and the corresponding Bayes predictor $f_\phi^*(x)=\arg\min_z\mathbb{E}\big[D_\phi(Y,z)\mid X=x\big]$. As we saw in Section 2, a fundamental property of Bregman divergences implies that $f_\phi^*(x)=\mathbb{E}[Y\mid X=x]$.

Similarly to CART, a Bregman tree approximates the optimal predictor by a piecewise constant function $f_T(x)=\sum_{j=1}^M c_j \mathbf{1}_{\{x\in t_j\}}$.For each region \(t_j\), the optimal constant predictor is $c_j
=
\arg\min_c
\mathbb E[D_\phi(Y,c)\mid X\in t_j]
=
\mathbb E[Y\mid X\in t_j]$. Its empirical counterpart is the sample mean $\hat c_j
=
\frac1{|t_j|}
\sum_{i:X_i\in t_j}Y_i$. Accordingly, the impurity associated with a node \(t\) is defined as $I_\phi(t)=\frac{1}{|t|}\sum_{i:X_i\in t}D_\phi(Y_i,\bar Y_t)$, where $\bar Y_t=\frac{1}{|t|}\sum_{i:X_i \in t}Y_i$, which generalizes the classical within-node variance criterion used in CART. As in classical CART, the resulting estimator can be written as a local averaging procedure,
\[
\hat f_n(x)
=
\sum_{i=1}^n W_{n,i}(x)Y_i,
\]
where the weights are induced by the partition generated by the tree and satisfy $\sum_i W_{n,i}(x)=1$.

The construction of a Bregman Tree therefore differs from classical CART only through the impurity measure. Splits are selected by maximizing the reduction in Bregman impurity, while prediction within each terminal node is given by the sample mean. The theoretical analysis developed in this section relies on suitable assumptions on the generating function $\phi$. In particular, assuming that $\phi$ is strongly convex, we derive a nontrivial lower bound for the split gain, expressed in terms of the distance between the mean parameters associated with the child nodes. We also establish a consistency theorem for the resulting Bregman Tree estimator. 

\section{Main Results}

Having introduced the proposed Bregman Tree framework, we now establish its main theoretical properties. Recall that, for a node \(t\), the CART algorithm selects the split that maximizes the decrease in impurity

\[
\Delta(t)
=
I(t)
-
\Bigl(
p_{t_L} I(t_L)
+
p_{t_R} I(t_R)
\Bigr),
\]
where $p_{t_L}
=
\frac{|t_L|}{|t|},
\,
p_{t_R}
=
\frac{|t_R|}{|t|}$, and $I(t)$ denotes the impurity associated with node $t$. Motivated by the correspondence between exponential families and Bregman divergences, we generalize this criterion by replacing the classical impurity measure with a Bregman divergence generated by a differentiable strictly convex function \(\phi\).

\subsection{Split Criterion}

For a node \(t\), define the Bregman impurity by
\[
I_\phi(t)
=
\frac{1}{|t|}
\sum_{i\in t}
D_\phi(y_i,\mu_t),
\]
where $\mu_t
=
\arg\min_{\mu}
\sum_{i\in t}
D_\phi(y_i,\mu)$ is the Bregman barycenter of the observations contained in the node. The impurity reduction associated with a split
\(t=t_L\cup t_R\)
is then

\[
\Delta_\phi(t)
=
I_\phi(t)
-
\Bigl(
p_{t_L}I_\phi(t_L)
+
p_{t_R}I_\phi(t_R)
\Bigr).
\]

The following theorem provides the fundamental decomposition underlying the proposed methodology.

\begin{theorem}[Bregman split decomposition]
Let \(t=t_L\cup t_R\) be a split of a node \(t\). Then

\[
\Delta_\phi(t)
=
\frac{|t_L|}{|t|}
D_\phi(\mu_{t_L},\mu_t)
+
\frac{|t_R|}{|t|}
D_\phi(\mu_{t_R},\mu_t),
\]

\end{theorem}

\begin{proof}
 The result follows to the three-point identity for Bregman divergences:  $$D_{\phi}(y,\mu)=D_{\phi}(y,\mu_A)+D_{\phi}(\mu_A,\mu)+\langle y-\mu_A,\nabla \phi(\mu_A)-\nabla \phi(\mu)\rangle$$  

 Summing over the observations in $t_L$

 $$\sum \limits_{i\in t_L}D_{\phi}(y_i,\mu)=\sum \limits_{i\in t_L}D_{\phi}(y_i,\mu_{t_L})+n_{t_L}D_{\phi}(\mu_{t_L},\mu)+\sum \limits_{i\in t_L}\langle y_i-\mu_{t_L},\nabla \phi(\mu_{t_L})-\nabla \phi(\mu)\rangle$$  as $\sum \limits_{i\in t_L}y_i-\mu_{t_L}=0$ we have that 
 $$\sum \limits_{i\in t_L}D_{\phi}(y_i,\mu)=\sum \limits_{i\in t_L}D_{\phi}(y_i,\mu_{t_L})+n_{t_L}D_{\phi}(\mu_{t_L},\mu)$$
An analogous identity holds for $t_R$: 
 $$\sum \limits_{i\in t_R}D_{\phi}(y_i,\mu)=\sum \limits_{i\in t_R}D_{\phi}(y_i,\mu_{t_R})+n_{t_R}D_{\phi}(\mu_{t_R},\mu)$$

 Summing on all the points of node $t=t_L\cup t_R$, we have:
 $$\sum_{i=1}^nD_{\phi}(y_i,\mu)=\sum_{i\in t_L}D_{\phi}(y_i,\mu)+\sum_{i\in t_R}D_{\phi}(y_i,\mu)=\sum \limits_{i\in t_L}D_{\phi}(y_i,\mu_{t_L})+n_{t_L}D_{\phi}(\mu_{t_L},\mu)+\sum \limits_{i\in t_R}D_{\phi}(y_i,\mu_{t_R})+n_{t_R}D_{\phi}(\mu_{t_R},\mu)$$

 $$\sum_{i=1}^nD_{\phi}(y_i,\mu)=\sum \limits_{i\in t_L}D_{\phi}(y_i,\mu_{t_L})+\sum \limits_{i\in {t_R}}D_{\phi}(y_i,\mu_{t_R})+n_{t_L}D_{\phi}(\mu_{t_L},\mu)+n_{t_R}D_{\phi}(\mu_{t_R},\mu)$$

 Dividing both sides by $n$ yields:

 $$I(\text{father})=I(\text{childs})+\frac{n_{t_L}}{n}D_{\phi}(\mu_{t_L},\mu)+\frac{n_{t_R}}{n}D_{\phi}(\mu_{t_R},\mu)$$


\end{proof}

The theorem shows that the gain produced by a split is completely determined by the Bregman discrepancies between the child-node barycenters and the parent-node barycenter, weighted by the relative node sizes. It has a clear interpretation: the gain of a split is exactly the weighted Bregman divergence between the children means and the parent mean. A split is informative when the children means diverge substantially from the parent mean, measured in the geometry that is natural for the chosen distribution family. Then, the theoretically optimal split consists of finding the split that maximizes the separation between the barycenters in Bregman metric.

Consequently, every split produces a nonnegative impurity reduction, so that \(\Delta_\phi(t)\) defines a valid splitting criterion.
\begin{corollary}
    
\[
\Delta_\phi(t)\ge 0.
\]

\end{corollary}

 Note that $\mu_t=\frac{|t_L|}{|t|}\mu_{t_L}+\frac{|t_R|}{|t|}\mu_{t_R}$. We use this fact in the following corollary. If \(\phi\) is \(\lambda\)-strongly convex, that is
 $D_\phi(x,y)
\ge
\frac{\lambda}{2}
\|x-y\|^2$, it yields a quantitative lower bound, implying that splits separating sufficiently distinct subpopulations necessarily generate a nontrivial gain and we have the following result:

\begin{corollary}
\[
\Delta_\phi(t)
\ge
\frac{\lambda}{2}
\frac{|t_L||t_R|}{|t|^2}
\,
\|\mu_{t_L}-\mu_{t_R}\|^2.
\]

\end{corollary}
\begin{proof}
    This lower bound is a direct application of $\lambda$-strong convexity of $\phi$: for all  $x,y$
\[
\phi(x) \ge \phi(y) + \langle \nabla \phi(y), x-y \rangle + \frac{\lambda}{2}\|x-y\|^2.
\]

$$ \frac{|t_L|}{|t|}D_{\phi}(\mu_{t_L},\mu_t)+\frac{|t_R|}{|t|}D_{\phi}(\mu_{t_R},\mu_t)\geq \frac{\lambda}{2}\left( \frac{|t_L|}{|t|}||\mu_{t_L}-\mu_t||^2+\frac{|t_R|}{|t|}||\mu_{t_R}-\mu||^2\right)$$

As $\mu=\frac{|t_L|}{|t|}\mu_{t_L}+\frac{|t_R|}{|t|}\mu_{t_R}$ and substituting into the expression above, we obtain:

\[
\frac{|t_L|}{|t|}D_\phi(\mu_{t_L},\mu_t)
+
\frac{|t_R|}{|t|}D_\phi(\mu_{t_R},\mu_t)
\ge
\frac{\lambda}{2}\frac{|t_L|}{|t|}\frac{|t_L|}{|t|}\|\mu_{t_L} - \mu_{t_R}\|^2.
\]

\end{proof}

\subsection{A General Convex Framework for Tree Impurities}

The construction of impurity measures based on convex functions is closely related to several classical ideas in statistics, information theory, and convex analysis. The population expression of the impurity measure is $I_{\phi}(t)=\mathbb{E}(D_{\phi}(Y,\mu_t))$ where $\mu_t=\mathbb{E}(Y|X\in t)$. A straightforward calculus gives that $$I_{\phi}(t)=\mathbb{E}(\phi(Y)|X\in t)-\phi(\mathbb{E}(Y|X\in t))$$

Starting from the impurity reduction induced by a split of node $t$ into children $t_L$ and $t_R$,

$$\Delta_\phi(t)
=
I_\phi(t)
-
\Big(
p_{t_L} I_\phi(t_L)
+
p_{t_R} I_\phi(t_R)
\Big)$$

where $p_{t_L}=P(X\in t_L\mid X\in t)$ and $
p_{t_R}=P(X\in t_R\mid X\in t)$ and substituting the definition of $I_\phi$ into $\Delta_\phi(t)$ yields:

$$\Delta_\phi(t)
=
\Big(
\mathbb{E}[\phi(Y)\mid t]
-
\phi(\mathbb{E}[Y\mid t])
\Big) -
p_{t_L}
\Big(
\mathbb{E}[\phi(Y)\mid t_L]
-
\phi(\mathbb{E}[Y\mid t_L])
\Big)-
p_{t_R}
\Big(
\mathbb{E}[\phi(Y)\mid t_R]
-
\phi(\mathbb{E}[Y\mid t_R])
\Big)$$
Rearranging terms,

$$\Delta_\phi(t)
=
\Big(
\mathbb{E}[\phi(Y)\mid t]
-
p_{t_L}\mathbb{E}[\phi(Y)\mid t_L]
-
p_{t_R}\mathbb{E}[\phi(Y)\mid t_R]
\Big)
+
p_{t_L}\phi(\mu_L)
+
p_{t_R}\phi(\mu_R)
-
\phi(\mu_t),$$
where $\mu_t=\mathbb{E}[Y\mid t], \mu_L=\mathbb{E}[Y\mid t_L], \mu_R=\mathbb{E}[Y\mid t_R]$

Then, by the law of total expectation 
$\mathbb{E}[\phi(Y)\mid t]
=
p_{t_L}\mathbb{E}[\phi(Y)\mid t_L]
+
p_{t_R}\mathbb{E}[\phi(Y)\mid t_R]$, the first term vanishes, giving

$$\Delta_\phi(t)
=
p_{t_L}\phi(\mu_L)
+
p_{t_R}\phi(\mu_R)
-
\phi(\mu_t)$$

Since $\mu_t = p_{t_L}\mu_L +p_{t_R}\mu_R$, we obtain

$$\Delta_\phi(t)
=
p_{t_L}\phi(\mu_L)
+
p_{t_R}\phi(\mu_R)
-
\phi\!\left(
p_{t_L}\mu_L+p_{t_R}\mu_R
\right).$$

Since $M=\mathbb{E}[Y\mid T]$, the random variable $M$ takes the values $\mu_L=\mathbb{E}[Y\mid t_L]$ and 
$\mu_R=\mathbb{E}[Y\mid t_R]$ according to whether $T=t_L$ or $T=t_R$. Hence

$$\mathbb{E}[\phi(M)]
=
p_{t_L}\phi(\mu_L)
+
p_{t_R}\phi(\mu_R),\,\,\,\text{and}\,\,\,\mathbb{E}[M] =p_{t_L}\mu_L +p_{t_R}\mu_R= \mu_t$$


and consequently, $\Delta_\phi(t) =
\mathbb{E}[\phi(M)]-\phi\!\left(\mathbb{E}[M]\right)$, which is precisely the Jensen gap associated with the convex function $\phi$ (and therefore, the nonnegativity of the impurity reduction follows directly from the convexity of $\phi$). We have proved the following result:

\begin{theorem}\label{teo:jensen-gap}
Let $\phi:\mathbb{R}\to\mathbb{R}$ be a differentiable convex function, the Bregman divergence $D_{\phi}$ generated by $\phi$ and let

\[
I_\phi(t)
=\mathbb{E}(D_{\phi}(Y,\mu_t))=
\mathbb{E}\!\left[\phi(Y)\mid X\in t\right]
-
\phi\!\left(
\mathbb{E}[Y\mid X\in t]
\right)
\]

the impurity of node $t$. Consider a split $t=t_L\cup t_R$ and let
$T\in\{t_L,t_R\}$ be the child node containing the observation. Defining $M=\mathbb{E}(Y|T)$ and $\mu_t=\mathbb{E}(Y|X\in t)$, then the impurity reduction is

$$
\Delta_\phi(t)
=I_\phi(t)
-
\sum_{u\in\{t_L,t_R\}}
P(T=u)\,I_\phi(u)=
\mathbb{E}\!\left[
D_\phi(M,\mu_t)
\right]$$
\end{theorem}

Hence, the gain associated with a split is itself an expected Bregman divergence measuring the separation between the child-node means and the parent-node mean. Theorem \ref{teo:jensen-gap} shows that impurity reduction is exactly the Jensen gap associated with the convex generator $\phi$. Consequently, split selection in Bregman Trees can be interpreted as searching for partitions that maximize the departure from Jensen equality. This provides a purely convex-analytic characterization of recursive partitioning and unifies variance reduction, entropy reduction, and exponential-family deviance criteria within a common framework.

 For any convex differentiable generator $\phi$, the node impurity can be written as
\[
I_\phi
=
\mathbb E[\phi(Y)]
-
\phi(\mu)
=
\mathbb E[D_\phi(Y,\mu)],
\qquad
\mu=\mathbb E[Y].
\]
This framework unifies several classical tree criteria. In particular, $\phi(\mu)=\frac{\mu^2}{2}$ yields the usual variance reduction criterion, $\phi(\mu) =
\mu\log\mu-\mu$ yields the Poisson deviance criterion, $\phi(\mu)= -\log\mu$ yields the Gamma deviance criterion, and $\phi(\mu)
=
\mu\log\mu+(1-\mu)\log(1-\mu)$ yields the KL deviance (cross-entropy) criterion. Table~\ref{tab:bregman_impurities} shows several classical impurity measures arising from this construction.

\begin{table}[ht]
\centering
\small
\begin{tabular}{lllll}
\toprule
Criterion &
Generator $\phi$ &
$D_\phi(y,\mu)$ &
Associated loss &
Node impurity $I_\phi$
\\
\midrule

ANOVA
&
$\mu^2$
&
$(y-\mu)^2$
&
Squared error
&
$\operatorname{Var}(Y)$
\\[0.2cm]

Gini
&
$\|p\|^2$
&
$\|p-q\|^2$

&
&
$1-\sum_{k=1}^K p_k^2$
\\[0.2cm]

Entropy
&
$\displaystyle \sum_{k=1}^K p_k\log p_k$

&
$D_{KL}(p,q)
 =
 \sum_{k=1}^K
 p_k\log\frac{p_k}{q_k}$
&
Cross-entropy / KL
&
$\displaystyle
-\sum_{k=1}^K p_k\log p_k
$
\\[0.35cm]

Poisson
&
$\mu\log\mu-\mu$
&
$\displaystyle
y\log\frac{y}{\mu}
-y+\mu
$
&
Poisson deviance
&
$\displaystyle
\mathbb E[Y\log Y]
-\mu\log\mu
$
\\[0.35cm]

\\

\bottomrule
\end{tabular}
\caption{
Classical tree impurity criteria as special cases of Bregman impurities.
In all cases,
$
I_\phi
=
\mathbb E[\phi(Y)]
-
\phi(\mu)
=
\mathbb E[D_\phi(Y,\mu)]$.
}
\label{tab:bregman_impurities}
\end{table}

From this perspective, regression tree impurities, GLM deviances, entropy-based criteria, and variance reduction can all be viewed as special cases of a single principle $I_\phi(t)
=
\mathbb{E}[\phi(Y)\mid t]
-
\phi\!\left(
\mathbb{E}[Y\mid t]
\right)$, where \(\phi\) is an arbitrary convex function. To the best of our knowledge, the explicit formulation of regression tree construction through arbitrary convex generators \(\phi\), together with its interpretation in terms of expected Bregman divergences and exponential-family deviances, has not been systematically developed in the tree-learning literature.

\subsection{Consistency}

As mentioned in the previous section, there is no general consistency theorem for general CART, but there exists a result in a particular setting due to Devroye, Gyorfi and Lugosi (Theorem \ref{DGLtheorem}). For Bregman trees, the corresponding Bregman risk is $$\mathcal{R}(f)=\mathbb{E}\big[D_\phi(Y,f(X))\big]$$ and the Bayes predictor satisfies $f^*(x)=\arg\min_z\mathbb{E}\big[D_\phi(Y,z)\mid X=x\big]=\mathbb{E}[Y\mid X=x]$
We have that $\mathcal{R}(\hat f_n)-\mathcal{R}(f)=\mathbb{E}\bigl(D_{\phi}(f^*(X),\hat{f_n}(X))\bigr)$.

\begin{theorem}[Consistency of Bregman Trees]
Let \((X,Y)\in\mathbb{R}^d\times\mathbb{R}\), and let
\(\{(X_i,Y_i)\}_{i=1}^n\) be an i.i.d. sample. Let \(\phi:\mathbb{R}\to\mathbb{R}\) be a differentiable strictly convex function, and $D_{\phi}$ the associated Bregman divergence. Assume moreover that \(\phi\) is \(L\)-smooth, namely $D_\phi(x,y)
\leq
\frac{L}{2}||x-y||^2$. Let $\mathcal{P}_n$ be a partition of $\mathbb{R}^d$. For $x \in \mathbb{R}^d$, let $A_n(x)$ be the cell of $\mathcal{P}_n$ containing $x$, and define
\[
\hat f_n(x) = \frac{1}{N_n(x)} \sum_{i: X_i \in A_n(x)} Y_i,
\quad
N_n(x) = \#\{i : X_i \in A_n(x)\}.
\]

Assume that:
\begin{enumerate}
    \item $\mathrm{diam}(A_n(x)) \xrightarrow{P} 0$ (Cell shrinkage),
    \item $N_n(x) \xrightarrow{P} \infty$ (Neighborhood growth),
    \item $\mathbb{E}[Y^2] < \infty$.
\end{enumerate}

Then the excess Bregman risk converges to zero:

\[
\mathcal R_\phi(\hat f_n)-\mathcal R_\phi(f^*)
\to 0,
\]

\end{theorem}

\begin{proof}
    
If $\phi$ is L-smooth (i.e., its gradient is L-Lipshitz or $||\phi''||_{\infty}\leq L)$ we have: $$\mathbb{E}\bigl(D_{\phi}(f^*(X),\hat{f_n}(X))\bigr) \leq \frac{L}{2}\mathbb{E}\bigl((f^*(X)-\hat{f_n}(X))^2\bigr)$$ and the usual consistency theorem DGL holds in this case. By \(L\)-smoothness,

\[
\mathcal R_\phi(\hat f_n)-\mathcal R_\phi(f^*)
\leq
\frac{L}{2}
\,
\mathbb E\Big[
\big(
\hat f_n(X)-f^*(X)
\big)^2
\Big],
\]
Then, by Theorem \ref{DGLtheorem} $
E[(f^*(X)-\hat f_n(X))^2]
\to 0$ and therefore

\[
\mathcal R_\phi(\hat f_n)-\mathcal R_\phi(f^*)
\le
\frac L2
E[(f^*(X)-\hat f_n(X))^2]
\to 0.
\]

This complete the proof.
\end{proof}

\begin{remark}
The smoothness assumption is automatically satisfied for the quadratic
generator and, more generally, for any twice differentiable generator
whose second derivative is bounded on the relevant range of mean
parameters. In particular, the generators associated with Poisson,
Gamma, Exponential, and Bernoulli models are locally smooth and locally
strongly convex on compact subsets of their domains.
\end{remark}

\section{Numerical Experiments}
We now illustrate the performance of the proposed Bregman trees through a series of Monte Carlo simulations. The goal of these experiments is to compare the classical CART tree, based on impurity measures available in {\sf rpart}  with Bregman Trees constructed using divergences naturally adapted to the underlying exponential-family responses.
A Monte Carlo study was conducted across five data generating processes:
Exponential, Gamma, Inverse Gaussian, Poisson, and Gaussian. In each scenario, 100
independent replications were run and out-of-sample metrics were averaged across replications.
The experiments are designed to evaluate whether adapting the impurity criterion to the underlying probabilistic structure of the data can improve predictive performance.

For each simulated dataset, we construct both a standard CART tree and a Bregman tree with cost-complexity pruning. To ensure a fair comparison, both CART and Bregman Trees were pruned using cost-complexity pruning with ten-fold cross-validation. For the proposed method, the validation error was measured using the Bregman divergence corresponding to the underlying loss function. The final subtree was selected according to the one-standard-error (1-SE) rule of \cite{breiman1984classification}, namely, the smallest subtree whose estimated prediction error was within one standard error of the minimum cross-validated error. Their predictive performances are then compared using both the corresponding deviance and the classical mean squared error (MSE). Although the Mean Squared Error (MSE) is a valid measure of prediction error in general, it is not the appropriate metric for evaluating models under the Inverse Gaussian distribution. The MSE treats prediction errors uniformly regardless of the magnitude of the conditional mean, implicitly assuming constant variance. It is reported as a secondary metric to verify that the Bregman tree does not sacrifice performance under the alternative criterion. Minimizing the deviance is equivalent to maximizing the log-likelihood of the model, establishing it as the statistically principled criterion for any member of the exponential family. Additionally, the deviance is the quantity that the Bregman tree explicitly optimizes at every stage of the pipeline — splits, pruning, and model selection — so evaluating predictive performance under this criterion is the only way to compare the two methods on equal and meaningful terms. We use $n=1000$ observations with a classical 2/3 to train and 1/3 to test. Repeating the procedure over 100 Monte Carlo replications allows us to evaluate not only average performance but also the stability and robustness of the proposed method.

To ensure that observed differences in predictive performance are attributable
solely to the choice of splitting criterion rather than to differences in tree
complexity, both methods were fitted under identical stopping rules:
\texttt{minsplit} $= 40$ and \texttt{minbucket} $= 20$, corresponding to $4\%$ and $2\%$ of the training sample respectively. These values ensure stable mean estimates at each terminal node,
which is particularly important under heavy-tailed distributions where
the sample mean converges more slowly than under Gaussian assumptions.
The divergences considered in the experiments arise from Legendre generators associated with classical exponential-family models. For the Gaussian, Poisson, Gamma, Exponential and Inverse Gaussian families, the corresponding generators are twice continuously differentiable on the interior of their domains. Moreover, on every compact subset bounded away from the boundary, they are strongly convex and possess Lipschitz-continuous gradients. These local regularity conditions are sufficient for the asymptotic theory developed in Section 5.

\begin{enumerate}
    \item {\bf First Simulation: Exponential model.} We consider a response variable $Y \sim exp(\lambda), 
        p_{Y}(y)=\lambda e^{-\lambda y}$. The  associated convex Legendre conjugate is $\phi(\mu)=-\log \mu -1$ and the Bregman divergence is the Itakura-Saito divergence $D_{\phi}(y,\mu)=\frac{y}{\mu}-\log\left(\frac{y}{\mu} \right)-1$.

      The data are generated as follows
     $X_1, X_2 \sim \mathcal{U}(0,1)$, $\epsilon \sim \mathcal{N}(0,1)$ and the mean structure is
   $\mu(X) = \exp(1+X_1+2X_2)$ and response
    $Y|X \sim \text{Exponential}(1/\mu)$.

    For the splitting criterion, we compare CART for regression with the quadratic loss of {\sf anova}  and for the Bregman tree Itakura--Saito divergence.
The Bregman Tree uses the Itakura--Saito (IS) divergence
as its splitting criterion, which is the natural Bregman divergence for the Gamma and
Exponential families. CART-anova minimises the MSE instead.

  Under an exponential data generating process, the Bregman Tree outperforms CART-anova
on both metrics and in the vast majority of individual replications. The IS deviance
advantage is substantial: the Bregman Tree reduces the mean IS deviance relative to CART-anova ($0.6487$ vs.\ $0.6994$). More strikingly, the Bregman
Tree also achieves a lower mean MSE ($359.81$ vs.\ $374.54$) and
wins on MSE in $69\%$ of replications, despite MSE being the criterion that CART-anova
directly optimises during training. This suggests that a distributionally appropriate splitting criterion may improve generalisation even on metrics that are not explicitly optimised during training.

\begin{table}[h]
\centering
\caption{Monte Carlo results: Exponential model.}
\begin{tabular}{lrrrr}
\hline
\textbf{Model} & \textbf{Mean IS Dev.} & \textbf{Mean MSE} & \textbf{Win rate (IS)} & \textbf{Win rate (MSE)} \\
\hline
Bregman Tree & \textbf{0.6487} & \textbf{359.8187} & \textbf{89\%} & \textbf{69\%} \\
CART-anova   & 0.6994          & 374.5418         & 11\%          & 31\%          \\
\hline
\end{tabular}
\end{table}

 \item {\bf Second Simulation: Gamma model.} The Exponential is a special case of the Gamma family, and the
IS divergence remains the natural splitting criterion for both. Increasing the shape
parameter $\alpha$ reduces the coefficient of variation of the response, making the
data less dispersed than in Simulation~1. We  generalize the previous simulation by considering a Gamma-distributed response $Y \sim \Gamma(k,\theta)$\footnote{The exponential distribution is a special case of Gamma distribution: $Exp(\lambda)=\Gamma(\alpha=1,\theta=\lambda)$}. The Bregman divergence is still the Itakura-Saito divergence $D_{\phi}(y,\mu)=\frac{y}{\mu}-\log\left(\frac{y}{\mu} \right)-1$

      The data are generated as follows
     $X_1, X_2 \sim \mathcal{U}(0,1)$, $\epsilon \sim \mathcal{N}(0,1)$ and the mean structure is
   $\mu(X) = \exp(1+X_1+2X_2)$ and response
   is $Y|X \sim \Gamma(k,\theta),
    \theta=\mu/k$ with shape parameter $k=5$.
For the splitting criterion, we compare CART for regression with the quadratic loss of {\sf anova}  and for the Bregman tree Itakura--Saito divergence.

\begin{table}[h]
\centering
\caption{Monte Carlo results: Gamma model with shape $\alpha=5$.}
\label{tab:sim_gamma}
\begin{tabular}{lrrrr}
\hline
\textbf{Model} & \textbf{Mean IS Dev.} & \textbf{Mean MSE}
               & \textbf{Win rate (IS)} & \textbf{Win rate (MSE)} \\
\hline
Bregman Tree & \textbf{0.1273} & \textbf{78.25} & \textbf{92\%} & \textbf{58\%} \\
CART-anova   & 0.1471          & 79.40         & 8\%           & 42\%           \\
\hline
\end{tabular}
\end{table}

The Bregman Tree is almost uniformly better under the IS deviance and remains slightly superior under MSE.
The mean IS deviance of Bregman Trees is $0.1273$  and for CART $0.1471$, and the mean MSE
 results are $78.25$  and $79.40$ respectively. The win rate reaches 92\% under IS deviance and 58\% under MSE, indicating a clear advantage for the Bregman Tree under the distributionally appropriate loss, while the difference in MSE is more modest. The stronger dominance relative to Simulation~1 is consistent with the
reduced dispersion of the Gamma ($\alpha=5$) response: with less noise in the data, the
alignment between the splitting criterion and the true data generative process becomes more decisive, and
the misspecification penalty paid by CART-anova is more clearly reflected in the results.

\item {\bf Third Simulation: Inverse Gaussian.} The response was generated as
$Y \mid \mathbf{X} \sim \mathrm{InvGauss}(\mu, \lambda)$, with $\mu = e^{1 + X_1 + 2X_2}$
and dispersion parameter $\lambda = 5$. The Bregman Tree uses the Inverse Gaussian
divergence $D_\phi(y,\mu) = (y-\mu)^2 / (2\mu^2 y)$, associated with the generator
$\phi(\mu) = 1/\mu$.

\begin{table}[h]
\centering
\caption{Monte Carlo results: Inverse Gaussian model (100 replications).}
\label{tab:sim_ig}
\begin{tabular}{lrrrr}
\hline
\textbf{Model} & \textbf{Mean IG Dev.} & \textbf{Mean MSE}
               & \textbf{Win rate (IG Dev.)} & \textbf{Win rate (MSE)} \\
\hline
Bregman Tree & \textbf{0.1133} & \textbf{1{,}689.46} & \textbf{63\%} & 49\% \\
CART-anova   & 0.1173         & 1{,}698.75          & 37\%          &\textbf{51\%}          \\
\hline
\end{tabular}
\end{table}

The Bregman Tree exhibits a modest advantage under the Inverse Gaussian deviance, while both methods perform nearly identically in terms of MSE. The mean IG deviance is reduced to $0.1133$
vs. $0.1173$, and also the mean MSE  ($1{,}689.46$ vs.\ $1{,}698.75$).  The Inverse Gaussian distribution has
heavier tails and a more complex variance structure than the Gamma, which may reduce the
effective advantage of using a correctly specified criterion. The Bregman Tree still exhibits a small but systematic advantage under the Inverse Gaussian deviance, while the two methods remain essentially tied under MSE.

\begin{figure}[h]
  \centering
  \begin{subfigure}[b]{0.32\textwidth}
    \centering
    \includegraphics[width=\textwidth]{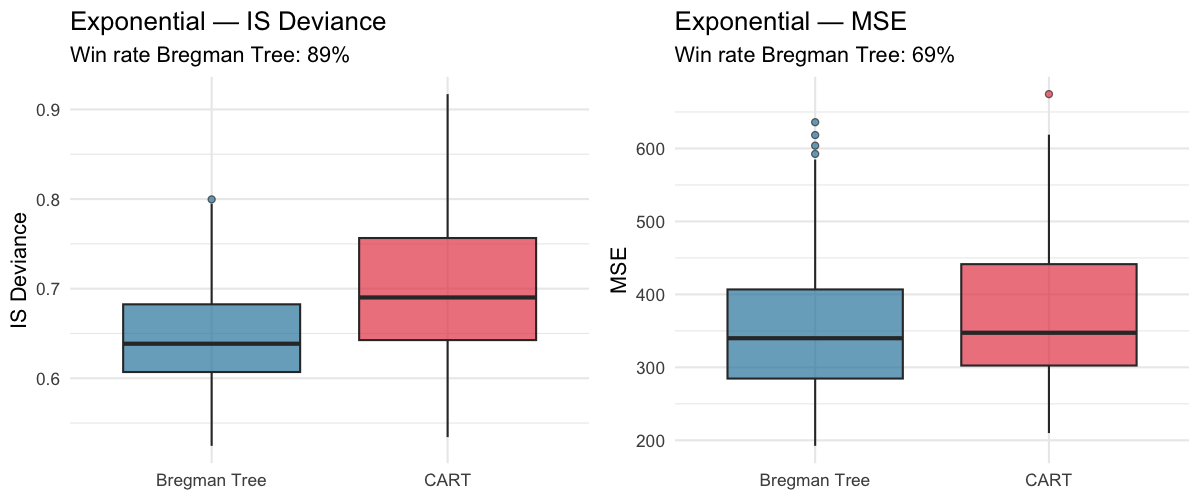}

    \caption{First Simulation: exponential model}
    \label{fig:fig1}
  \end{subfigure}
  \hfill
  \begin{subfigure}[b]{0.32\textwidth}
    \centering
    \includegraphics[width=\textwidth]{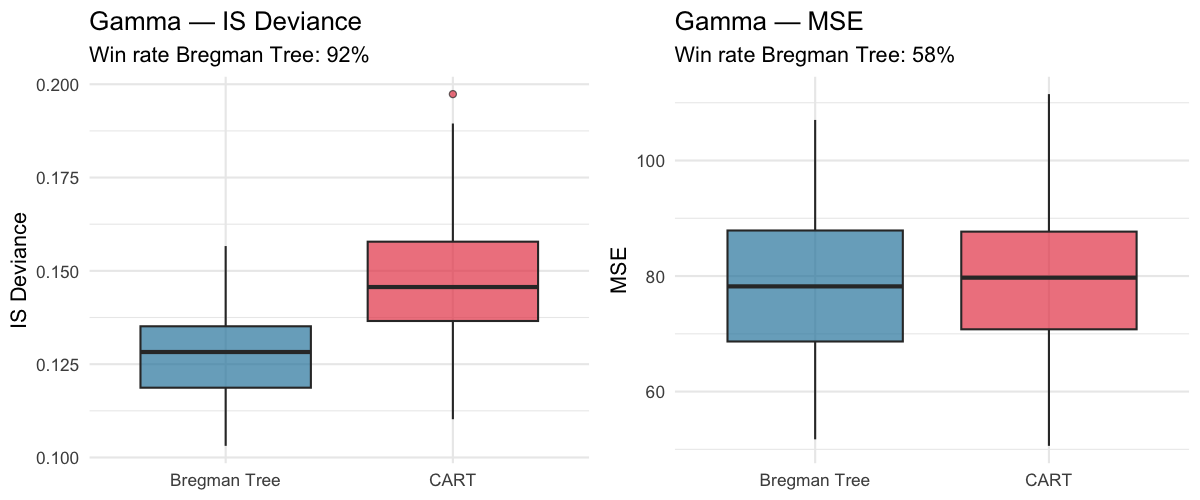}
    \caption{Second simulation: Gamma model}
    \label{fig:fig2}
  \end{subfigure}
  \hfill
  \begin{subfigure}[b]{0.32\textwidth}
    \centering
    \includegraphics[width=\textwidth]{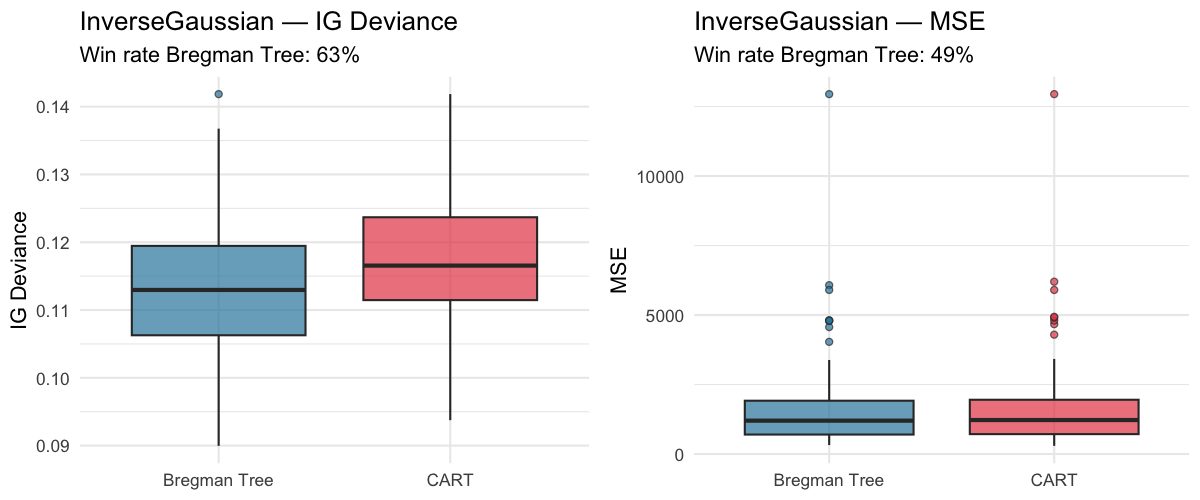}
    \caption{Third simulation: Inverse Gaussian model}
    \label{fig:fig3}
  \end{subfigure}
 \caption{Bregman Tree fitted under distributional losses unavailable in \texttt{rpart}:
Exponential (left), Gamma $(\alpha=5)$ (centre), and Inverse Gaussian $(\lambda=5)$ (right).
Each model uses the Bregman divergence associated with its respective data generating process.}
\label{fig:bregman_surfaces}
\end{figure}

\item {\bf Fourth Simulation: Poisson model.}
We consider a Poisson response variable. The  associated convex Legendre conjugate is $\phi(\mu)=\mu\log \mu -\mu$ and the Bregman divergence is  $D_{\phi}(y,\mu)=y\log \left(\frac{y}{\mu}\right)-(y -\mu)$

      The data are generated as follows
     $X_1, X_2 \sim \mathcal{U}(0,1)$, $\epsilon \sim \mathcal{N}(0,1)$ and the mean structure is $\mu(X) =\left\{\begin{array}{cc} \exp(1+2X_2)& X_1<0.5\\ \exp(2+X_2)&X_1 \geq 0.5\end{array}\right.$ with response
    $Y|X \sim \mathcal{P}(\mu)$.

    In this
scenario, both competing methods optimise the same criterion: the Bregman Tree uses the
KL divergence $D_\phi(y,\mu) = y\log(y/\mu) - y + \mu$, which is exactly the loss
minimised natively by the Poisson method implemented in CART. Accordingly, we compare the proposed method with CART  using {\sf method='poisson'}.

\begin{table}[h]
\centering
\caption{Monte Carlo results: Poisson model (100 replications).}
\label{tab:sim_poisson}
\begin{tabular}{lrrrr}
\hline
\textbf{Model} & \textbf{Mean Dev.} & \textbf{Mean MSE}
               & \textbf{Win rate (Dev.)} & \textbf{Win rate (MSE)} \\
\hline
Bregman Tree & 0.5999 & 12.6023 & 43\% & 38\% \\
CART-Poisson & \textbf{0.5900} & \textbf{12.3788} & \textbf{57\%} & \textbf{62\%} \\
\hline
\end{tabular}
\end{table}

This simulation serves as a negative control for the study. Since both procedures optimise the same population criterion, no systematic difference should be expected. The empirical results are fully consistent with this prediction. The differences in mean deviance ($0.5999$
vs.\ $0.5900$) and mean MSE ($12.6023$ vs.\ $12.3788$) are negligible. The slight edge of CART-Poisson is consistent with
random variation rather than a structural difference. This result serves as an important
negative control: it confirms that the gains observed in Simulations~1--3 are genuinely
attributable to criterion alignment, and not to some incidental implementation advantage
of the Bregman Tree.

\item {\bf Fifth Simulation: Gaussian model.} We finally look at a gaussian response variable. The  associated convex Legendre conjugate is $\phi(\mu)=\frac{1}{2}u^2$ and the Bregman divergence is the quadratic divergence $D_{\phi}(y,\mu)=\frac{1}{2}(y-\mu)^2$. The data are generated as follows

     $X_1, X_2 \sim \mathcal{U}(0,1)$, $\epsilon \sim \mathcal{N}(0,1)$ and the mean structure is
   $\mu(X) = 1+X_1+2X_2$ with response
    $Y|X \sim \mathcal{N}(\mu(X),1)$

 As in Simulation~4, Bregman Trees and CART-anova
optimise the same objective.

\begin{table}[h]
\centering
\caption{Monte Carlo results: Gaussian model (100 replications).}
\label{tab:sim_gauss}
\begin{tabular}{lr}
\hline
\textbf{Model} & \textbf{Mean MSE} \\
\hline
Bregman Tree & 1.1301 \\
CART-anova   & \textbf{1.1229} \\
\hline
\end{tabular}
\end{table}

Consistent with the Poisson simulation, no meaningful difference is observed between the
two methods. The mean MSE values ($1.1301$ vs.\ $1.1229$) are virtually identical, and the
small advantage of CART-anova falls well within the range of Monte Carlo variability.
This result closes the logical argument of the simulation study: the Bregman Tree offers
no advantage when its splitting criterion coincides with that of the benchmark, and
conversely, it provides systematic and sometimes substantial gains when it is strictly
better aligned with the true data generating process.

\begin{figure}[h]
  \centering
  \begin{subfigure}[b]{0.45\textwidth}
    \centering
    \includegraphics[width=\textwidth]{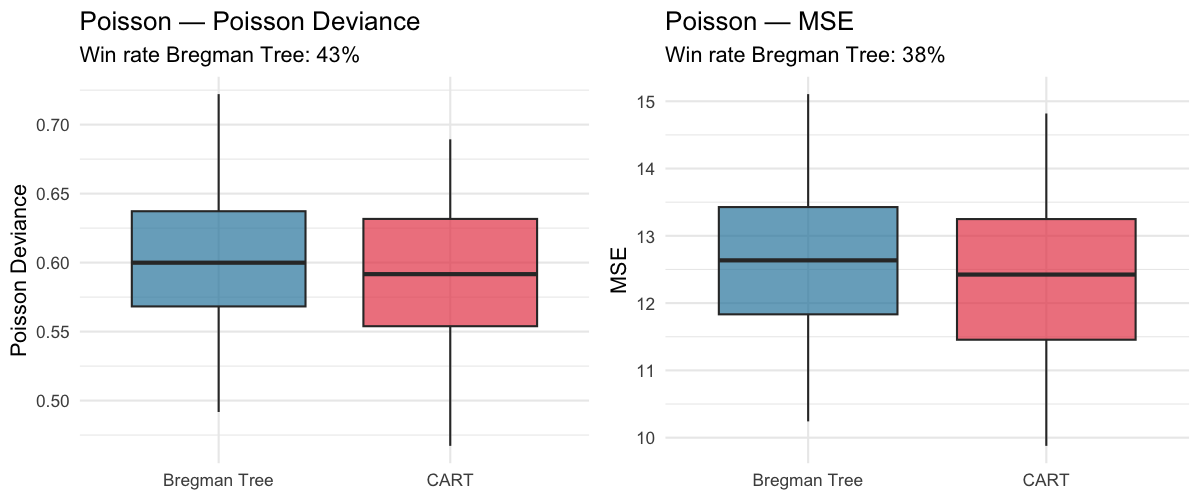}
    \caption{Poisson model.}
    \label{fig:sim_poisson}
  \end{subfigure}
  \hfill
  \begin{subfigure}[b]{0.45\textwidth}
    \centering
    \includegraphics[width=\textwidth]{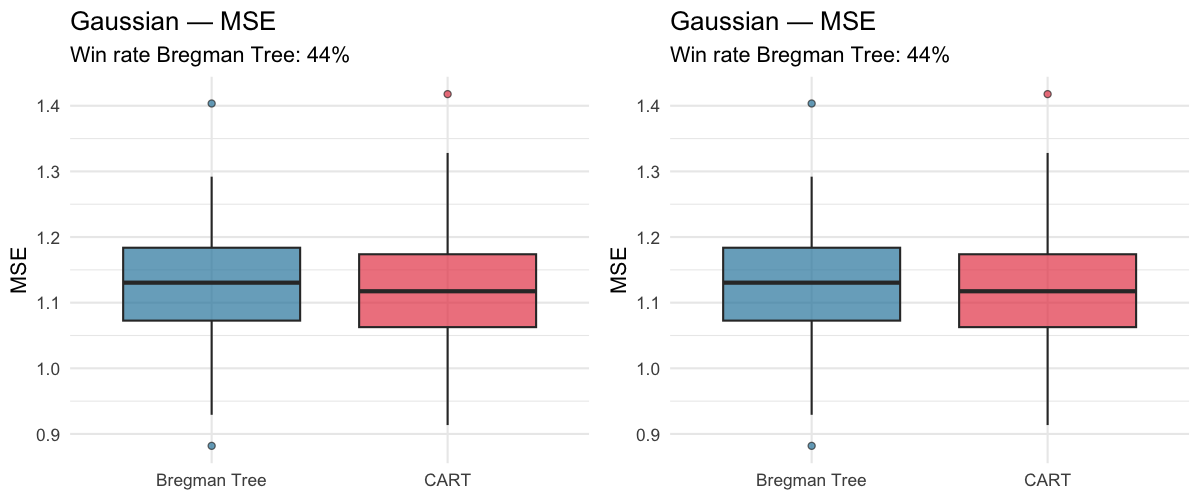}
    \caption{Gaussian model.}
    \label{fig:sim_gaussian}
  \end{subfigure}
  \caption{Bregman Tree vs.\ CART under Poisson (left) and Gaussian (right) data generating
  processes, where both methods optimise the same loss function. No significant differences
  in predictive performance are observed between the two approaches, confirming that the
  advantage of the Bregman Trees arises only when its splitting criterion is strictly better
  aligned with the true distribution than the competing method.}
  \label{fig:bregman_null}
\end{figure}

\end{enumerate}

Across all simulations, the advantage of the Bregman Tree appears precisely in settings where the classical CART criterion is misspecified relative to the underlying data generating process. The largest gains are observed for the Exponential and Gamma models, where the Itakura--Saito divergence provides a substantially better match to the conditional distribution than the quadratic loss. For the Inverse Gaussian model the improvement remains present but is more moderate, reflecting the increased variability of the response. In contrast, when both methods optimize the same population criterion, as in the Poisson and Gaussian settings, their predictive performances become essentially indistinguishable. This behaviour is entirely consistent with the theoretical interpretation of split gains as Bregman separations between conditional means and provides empirical support for the proposed framework.


\subsection{Real Data set}
We conclude by evaluating the proposed method on a real data set. The analysis is based on the \textit{insurance} dataset, which contains
1{,}338 records of medical insurance costs in the United States. The response variable
\texttt{charges} represents the amount billed to the patient by the insurer, expressed
in US dollars It is a continuous variable, strictly positive, and displays a pronounced right-skewed
distribution, which is consistent with a Gamma family model, which is well suited
for strictly positive, heteroscedastic responses where the variance tends to grow with the mean.
The explanatory variables are age, body mass index (BMI), smoking status, number of children, sex, and geographic region.

We fitted three single-tree models and compared them on a 70/30 train-test
split. Each model follows the same recursive partitioning logic but differs in the splitting
criterion it optimises, which makes the comparison particularly informative about the role
of the loss function in the tree-based modelling.

For the Bregman Tree, according to a possible Gamma distribution of $Y$, at each node, the split is chosen so as to minimise the Itakura--Saito (IS) divergence, which is the Bregman divergence naturally associated with the Gamma family through the convex generator $\phi(\mu) = -\log\mu$. Because this criterion is tailored to the distributional assumption of the response, one would expect it to yield better-calibrated predictions under a Gamma model. The final tree was selected by the 1-SE pruning rule. We compare this method with a the classical CART algorithm with anova, which greedily minimises the mean squared error (MSE) at each split. Although MSE is the standard choice for continuous responses, it implicitly assumes homoscedastic errors, an assumption that is unlikely to hold for insurance charges. The tree was pruned also by the 1-SE rule. Last we perform a variant of CART that replaces MSE with the Poisson deviance as the splitting criterion. The Poisson deviance penalises relative rather than absolute errors, making it more robust to the heavy tail of the response. As with the other models, pruning was carried out via the 1-SE rule.

The model performance on the test sample was assessed using
three complementary and accord metrics. The IS deviance $\mathrm{IS}(y,\hat{y})
  = \frac{1}{n}\sum_{i=1}^{n}
    \left(\frac{y_i}{\hat{y}_i} - \log\frac{y_i}{\hat{y}_i} - 1\right)$ measures how well predictions align with
the Gamma assumption; the Poisson deviance $\mathrm{DevPoisson}(y,\hat{y})
  = \frac{1}{n}\sum_{i=1}^{n}
    \left(y_i\log\frac{y_i}{\hat{y}_i} - y_i + \hat{y}_i\right)$ captures relative prediction accuracy, and finally we compare with the $\mathrm{MSE}(y,\hat{y})
  = \frac{1}{n}\sum_{i=1}^{n}(y_i - \hat{y}_i)^{2}$.

In Table~\ref{tab:results} we summarises the  performance of the
three models where we can see that the choice of splitting criterion has a
meaningful impact on predictive accuracy, particularly when the response distribution
departs from normality.

\begin{table}[h]
\centering
\caption{Out-of-sample performance of the three tree models on the test set.}
\label{tab:results}
\begin{tabular}{lrrrr}
\hline
\textbf{Model} & \textbf{Leaves} & \textbf{IS Deviance} & \textbf{Poisson Dev.} & \textbf{MSE} \\
\hline
Bregman Tree & 22 & 0.1212 & 1{,}895 & 21{,}343{,}132 \\
CART-anova   &  7 & 0.1443 & 2{,}097 & 22{,}314{,}072 \\
CART-Poisson & 27 & 0.1256 & 1{,}861 & 20{,}666{,}475 \\
\hline
\end{tabular}
\end{table}

In terms of tree size, CART-anova produces the most parsimonious structure, with only
7 terminal nodes. While this compactness is appealing from an interpretability standpoint,
it comes at a cost: CART-anova yields the worst performance across all three metrics,
suggesting that the model is underfitting the data. At the other extreme, CART-Poisson
grows the most complex tree with 27 leaves, yet it achieves the lowest MSE
($20{,}666{,}475$) and the best Poisson deviance ($1{,}861$), indicating that the
additional splits are capturing genuine structure in the data rather than noise. The
Bregman Tree sits between the two, with 22 leaves, and attains the lowest IS deviance
($0.1212$), which is the metric most directly aligned with its training objective and
with the assumed Gamma distribution of the response.

Regarding the IS deviance---the metric most relevant under a Gamma model---the Bregman
Tree outperforms both CART alternatives, as expected given that its splits are explicitly
designed to minimise this criterion. The gap between the Bregman Tree ($0.1212$) and
CART-anova ($0.1443$) is notable, representing roughly a $16\%$ relative reduction in
IS deviance. CART-Poisson ($0.1256$) closes much of this gap, which is not surprising
since the Poisson deviance also penalises relative errors and thus shares structural
similarities with the IS divergence.

In terms of MSE, CART-Poisson performs best ($20{,}666{,}475$), followed closely by the
Bregman Tree ($21{,}343{,}132$), while CART-anova lags behind ($22{,}314{,}072$). The
differences in MSE are substantial in absolute terms---over one million dollars in squared
error per observation on average---reflecting the heavy right tail of the \texttt{charges}
variable, where a small number of very high-cost cases can drive large squared residuals.
The fact that CART-anova, despite optimising MSE directly during training, does not achieve
the best test MSE is a reminder that a misspecified training criterion can lead to suboptimal
generalisation even on the metric it was designed to minimise.

Taken together, these findings suggest that explicitly accounting for the distributional
properties of the response---either through a Gamma-based criterion as in the Bregman Tree,
or through a relative-error criterion as in CART-Poisson---leads to better out-of-sample
performance than the standard CART based on squared-error loss. The Bregman Tree offers the best balance
between IS deviance and model interpretability, while CART-Poisson achieves the lowest
absolute and relative prediction errors at the cost of a more complex tree structure.

\section{Conclusion}

In this work, we studied impurity measures in CART decision trees through a Bregman divergence framework. This perspective unifies classical impurity criteria by expressing them in terms of Bregman divergences adapted to the statistical geometry of the problem. Our results show that many impurity criteria commonly used in machine learning — including squared loss, logistic deviance, Poisson deviance, and Kullback–Leibler divergence — can be understood within a single unified framework. In particular, the impurity gain associated with a split admits a natural geometric interpretation in terms of Bregman divergences between parent and child node barycenters.

Beyond the algorithmic perspective, the proposed approach establishes a rich connection between decision trees, convex analysis, and information geometry. Properties of the generating convex function, such as strong convexity and smoothness, directly influence the behavior of impurity gains and provide quantitative control on the geometry of the partitioning process. 

The numerical experiments support the theoretical findings. When the impurity criterion is aligned with the underlying data-generating process, Bregman trees can lead to improved performance over standard CART procedures. Conversely, when both methods optimise the same population criterion, their performances become essentially indistinguishable. These results suggest several possible directions for future work, including refined consistency results, statistical rates, extensions to classification settings, and ensemble methods based on non-Euclidean impurity criteria. More broadly, Bregman trees provide a flexible framework for adapting decision tree methods to the intrinsic geometry of the underlying statistical model.

\bibliographystyle{plainnat}
\bibliography{bibliography}

\newpage

\end{document}